\documentclass[final,12pt]{elsarticle}  

\makeatletter
\def\ps@pprintTitle{%
  \let\@oddhead\@empty
  \let\@evenhead\@empty
  \def\@oddfoot{\reset@font\hfil\thepage\hfil}%
  \let\@evenfoot\@oddfoot
}
\makeatother

\usepackage[T1]{fontenc}
\usepackage[utf8]{inputenc}
\usepackage{lmodern}
\usepackage{amsmath,amssymb,amsthm}
\DeclareMathOperator*{\argmin}{arg\,min}

\usepackage{graphicx}
\usepackage{booktabs}
\usepackage{tabularx}
\usepackage{array}
\usepackage{multirow}
\usepackage{longtable}
\usepackage{enumitem}
\usepackage{xcolor}
\definecolor{muted}{RGB}{100,116,139}
\definecolor{builtcol}{RGB}{4,120,87}
\definecolor{crimsonc}{RGB}{185,28,28}
\usepackage{algorithm}
\usepackage{algpseudocode}
\usepackage[section]{placeins}
\usepackage{pgfplots}
\pgfplotsset{compat=1.17}
\usepackage[hidelinks]{hyperref}

\newtheorem{proposition}{Proposition}
\newtheorem{definition}{Definition}

\theoremstyle{remark}
\newtheorem{remark}{Remark}

\begin{document}

\begin{frontmatter}

\title{Attribution Markets: A Fisher-Market Formulation for Fractional Credit Assignment Between Planned Tasks and Performed Actions}

\author{Salavat Ishbulatov \\
Independent researcher \\
\texttt{salavat@doplan.ai}}
\date{}

\begin{abstract}
Personal and organizational planning systems maintain two records that
drift apart: what was \emph{planned} (a task's effort budget) and what was
\emph{done} (a logged action's duration and description). Existing systems
bridge them with an exclusive, all-or-nothing link that strands genuinely
related but unlinked effort and reports false stalls on active goals. We
formulate the bridge as a \emph{quasi-linear Fisher market}: planned tasks
are budget-constrained buyers, performed actions are divisible goods, and a
fused text/structural/temporal signal sets each buyer's valuation. Two
market instruments -- a seller reserve price and a buyer cash option --
yield conservation, a hard budget cap, and a provable junk filter as
theorems. We extend the market with a concave \emph{completion utility}
discounting progress as a task nears its plan; standard convergence theory
for the market's algorithm does not transfer here, resolved by a
satiation-threshold fixed point with existence (Brouwer) and local
uniqueness under an explicit diagonal-dominance condition, validated
empirically on random and adversarial instances. A de-circularized,
multi-seed benchmark -- observed affinity corrupted independently of the
scored ground truth -- surfaces a genuine weak spot: the market's sharp,
zero-entropy equilibrium is more sensitive to affinity noise than
entropy-regularized optimal transport's permanently smoothed one. We
resolve this with a one-parameter entropy-regularized generalization
unifying the two, plus a noise-adaptive rule for its regularization
strength. We report full reproducibility parameters, discuss limitations
candidly, and relate the result to multi-touch attribution, optimal
transport, and online Fisher-market algorithms.
\end{abstract}


\end{frontmatter}

\section{Introduction}
\label{sec:intro}

\subsection{Motivation}

A planning system -- personal or organizational -- maintains two records
that are supposed to describe the same activity from two directions. The
\emph{plan side} holds tasks: a description, an effort budget, and a time
window. The \emph{result side} holds a log of what actually happened: a
description, a duration, a timestamp. Progress reporting requires a bridge
between the two: every logged unit of effort should be credited to some
task (or honestly to none), and a task's reported progress should be the
sum of the credit it received.

The bridge used in practice is almost always an \emph{exclusive link}: an
action is created under a task, in which case all of its duration counts
toward that task, or it is not, in which case none of it does. This
all-or-nothing rule fails in a specific and common way: work that is
genuinely related to a task -- preparatory reading, adjacent practice,
work in the same domain performed without explicitly opening the task --
is invisible to the bridge. The result is a system that reports zero
progress on a goal the person has, in fact, been advancing, and a review
process that opens with a false alarm.

\subsection{Contribution}

We propose to treat the bridge as a market: every performed action is a
divisible \emph{good}, every planned task is a budget-constrained
\emph{buyer}, and a fused similarity signal sets each buyer's valuation for
each good. The underlying object is not new -- it is the Fisher market of
Eisenberg and Gale \citep{eisenberg1959} -- but, to our knowledge, it has
not been applied to this attribution problem, and two instruments placed on
it (a seller reserve price and a buyer cash option) turn out to deliver
exactly the guarantees the problem needs. Concretely, we
(i)~define a formal \emph{attribution market} (Section~\ref{sec:model}) with
three properties proved as theorems -- conservation, a hard budget cap, and
a provable junk filter -- none of which the softmax or optimal-transport
baselines satisfy; (ii)~add a \emph{completion-seeking} extension
(Section~\ref{sec:completion}) whose saturating utility breaks the standard
convergence proof for the market's algorithm, show that the obvious naive
fix is vacuous, and resolve it with a
\emph{satiation-threshold} fixed point (Algorithm~\ref{alg:outerloop}) whose
existence we prove via Brouwer's theorem and whose local convergence we
prove under an explicit sufficient condition, both validated numerically;
(iii)~surface and explain a genuine \emph{weak spot}
(Section~\ref{sec:noise}) -- a de-circularized, multi-seed benchmark shows
the market is more sensitive to affinity noise than entropy-regularized
optimal transport, which we trace to the distinction between a permanently
regularized fixed point and a vanishing-regularization solution path and
resolve with a one-parameter entropic generalization and a noise-adaptive
rule for its strength; and (iv)~report a fully reproducible, de-circularized
synthetic evaluation (Section~\ref{sec:experiments}) with an explicit
reproducibility table, stated threats to validity, and a candid account of
where the guarantees are weaker than they first appear
(Section~\ref{sec:limits}).

This is, deliberately, a formulation-and-application paper rather than a
claim of new mathematics: the Fisher equilibrium, its computation by
proportional response dynamics, and the mean--variance machinery we draw on
are all classical. What is new is treating planned effort as the currency of
a market that clears logged effort -- a move that converts two informal
requirements (``do not over-credit a task,'' ``do not force unrelated work
onto the nearest task'') into theorems about a well-understood equilibrium,
together with the analysis of exactly where the formulation's natural
extensions break existing convergence theory and how those breaks are
repaired. In the tradition of applied operations-research and
decision-support work, we import an established equilibrium concept into a
new domain, analyze rigorously what transfers and what does not, and
validate the result. Two further extensions -- temporal dynamics under a
settlement horizon and a forward-looking market that steers future
planning -- are summarized in Section~\ref{sec:extensions} and developed in
full in a companion technical report \citep{doplan2026report}; the scope of
the present paper is the formulation, its guarantees, their computation, and
their honest empirical evaluation.

\subsection{Paper organization}

Section~\ref{sec:related} reviews related work. Section~\ref{sec:formulation}
gives the data model and the affinity fusion signal, including a concrete
fitting procedure. Section~\ref{sec:model} defines the base attribution
market and proves its guarantees. Section~\ref{sec:completion} develops
the completion-seeking extension and its convergence resolution, the
risk-aware valuation, and, in summary, the temporal and forward-looking
extensions. Section~\ref{sec:experiments} reports the de-circularized
empirical evaluation, including the noise-sensitivity weak spot it surfaces,
the theoretical explanation for it, and its resolution.
Section~\ref{sec:discussion} discusses when simpler alternatives suffice
(Section~\ref{sec:alternatives}), the broader thesis that planning is
describable as a market -- of which the Fisher model here is one instance,
opening a dynamic, uncertainty-aware, and learning-based program
(Section~\ref{sec:substrate}) -- and the model's limitations
(Section~\ref{sec:limits}). Section~\ref{sec:conclusion} concludes.

\section{Related work}
\label{sec:related}

The attribution problem sits at the intersection of several literatures,
none of which solves it alone; Table~\ref{tab:relmap} summarizes the
mapping we develop in this section.

\paragraph{Credit assignment and multi-touch attribution.} The general
credit-assignment problem -- given a composite outcome, which contributing
decisions deserve credit -- goes back to Minsky~\citep{minsky1961};
reinforcement learning develops its temporal form
\citep{sutton1984,pignatelli2024survey}.
The closest existing \emph{problem} to ours is multi-touch attribution in
marketing: a conversion is credited fractionally across advertising
touchpoints. Early data-driven models used bagged logistic regression
\citep{shao2011mta} and Markov-chain removal effects
\citep{anderl2016journey}; the field's recent turn is causal, reweighting
journeys to remove confounding before attributing \citep{yao2022causalmta}.
This literature supplies the fractional-credit framing but, without a budget or
a price system, offers no mechanism analogous to our budget-cap guarantee.

\paragraph{Soft assignment and optimal transport.} Mixture models fitted by
expectation--maximization \citep{dempster1977em} and attention
\citep{vaswani2017attention} are per-item fractional weighting schemes,
equivalent in our hierarchy (Section~\ref{sec:hierarchy}) to a single,
uncoupled Sinkhorn step. Optimal transport
\citep{kantorovich1942,cuturi2013sinkhorn,peyre2019computational}, and its
unbalanced generalization \citep{chizat2018unbalanced}, couples the columns
(actions) through conservation and the rows (tasks) through capacity;
low-rank and unbalanced solvers \citep{scetbon2023unbalanced} have made it
efficient at scale. This is the strongest \emph{price-free} baseline in our
comparison, and, as our experiments show (Section~\ref{sec:experiments}), it
is the more \emph{noise-robust} one -- a fact we explain theoretically in
Section~\ref{sec:noise}.

\paragraph{Market-based allocation.} The Fisher market equilibrium and its
convex-program characterization for linear utilities
\citep{eisenberg1959,nisan2007agt}, its computation by proportional response
dynamics \citep{zhang2011prd,birnbaum2011prd}, and its fairness
interpretation as proportionally fair \citep{kelly1998rate} and as the Nash
bargaining solution \citep{nash1950} are the mathematical core of our base
model. Approximate competitive equilibrium from equal incomes has been
deployed for course-seat assignment \citep{budish2011}, and pacing
equilibria are used for budget-constrained ad auctions
\citep{conitzer2022pacing}; both demonstrate that taking a real allocation
problem to a market equilibrium is an established, production-tested move,
which is the move we make here. Recent work extends Fisher-market
computation online, with regret and statistical-inference guarantees
\citep{gao2021pace,liao2022inference}; we discuss this as the natural path
to a streaming version of our model in Section~\ref{sec:extensions}.

\paragraph{Portfolio optimization.} The Eisenberg--Gale objective is,
formally, a budget-weighted sum of log-returns -- the objective of
log-optimal (Kelly) portfolio growth \citep{kelly1956,cover1991} -- and
Eisenberg and Gale's original derivation was as the pari-mutuel method for
betting markets, making "tasks as bettors staking budget on evidence" the
model's literal ancestry rather than a decorative analogy. Mean--variance
portfolio selection \citep{markowitz1952} and its robustification
\citep{black1992,ledoit2004} inform the risk-aware extension we use in
Section~\ref{sec:completion}.

\paragraph{Desktop task-context tracking and record linkage.} TaskTracer and
TaskPredictor \citep{dragunov2005,shen2006} associated low-level desktop
activity with declared tasks two decades ago, using hard classification --
exactly the all-or-nothing bridge whose failure motivates this paper.
Record linkage theory \citep{fellegi1969} supplies the log-linear
evidence-fusion template we use for the affinity signal
(Section~\ref{sec:affinity}); LLM-based entity matching
\citep{peeters2024llm} and standardized embedding benchmarks
\citep{muennighoff2023mteb} are the modern instantiation of that sensor.

\begin{table}[h]
\centering\small
\begin{tabularx}{\textwidth}{@{}lX@{}}
\toprule
\textbf{Requirement} & \textbf{Literature supplying it} \\
\midrule
Fractional, revisable credit & Credit assignment; multi-touch attribution \\
Meaning-based matching & Embeddings; task-context tracking \\
Fusing heterogeneous evidence & Record linkage (Fellegi--Sunter) \\
No double counting; capacities & Optimal transport \\
Budgets, contention, fairness & Fisher markets (Eisenberg--Gale) \\
Growth-optimality, diversification & Log-optimal / Kelly portfolio theory \\
Risk-awareness under noisy inputs & Mean--variance portfolio theory \\
\bottomrule
\end{tabularx}
\caption{Requirements of the attribution problem and the literature each is
drawn from. No single strand supplies all of them; the contribution of this
paper is the assembly and its analysis.}
\label{tab:relmap}
\end{table}

\section{Problem formulation}
\label{sec:formulation}

\subsection{Data model and notation}
\label{sec:notation}

At any point in time the system holds tasks $i=1,\dots,m$, each with a text
description, an effort budget $b_i>0$ (planned hours not yet settled), and a
planning window $[\sigma_i,\tau_i]$; and actions $j=1,\dots,n$, each with a
text description, a duration $d_j>0$ (logged hours), a timestamp $t_j$, and
an optional explicit link $\ell_j\in\{1,\dots,m\}\cup\{\varnothing\}$. A
distinguished index $i=0$, the \emph{float}, holds unattributed effort. The
object to be computed is a \emph{share matrix}
$W\in[0,1]^{(m+1)\times n}$ with $\sum_{i=0}^m w_{ij}=1$ for every $j$;
task progress is $P_i=\sum_j w_{ij}d_j$ and quality-adjusted progress is
$V_i=\sum_j q_{ij}w_{ij}d_j$, where $q_{ij}$ is the affinity defined next.

\subsection{The affinity signal}
\label{sec:affinity}

Following the record-linkage template \citep{fellegi1969}, we fuse three
independent evidence sources in log-odds space and pass the result through
a logistic link:
\begin{equation}
q_{ij} \;=\; \sigma\Big(
   \lambda_{\mathrm{sem}}\cos\!\big(E(x^{\mathrm T}_i),E(x^{\mathrm A}_j)\big)
 + \lambda_{\mathrm{link}}\,\mathbb{1}[\ell_j=i]
 + \lambda_{\mathrm{time}}\,\kappa(t_j;\sigma_i,\tau_i)
\Big)^{\gamma},
\label{eq:affinity}
\end{equation}
where $E$ is a text embedding model, $\kappa$ is a window kernel (full
weight inside the task's window, discounted outside), $\sigma(\cdot)$ is
the logistic function, and $\gamma>1$ sharpens moderate similarities toward
the extremes. Unlike the exploratory treatment in our earlier technical
report \citep{doplan2026report}, we here give a concrete, reproducible
fitting procedure for $\lambda=(\lambda_{\mathrm{sem}},\lambda_{\mathrm{link}},
\lambda_{\mathrm{time}})$.

\paragraph{Fitting procedure.} Given a labelled correction set
$\mathcal{D}=\{(i,j,y_{ij})\}$, where $y_{ij}\in\{0,1\}$ records a
user-confirmed match ($1$) or non-match ($0$) between task $i$ and action
$j$, $\lambda$ is fit by $\ell_2$-regularized logistic regression: writing
$s_{ij}(\lambda)$ for the pre-logistic score in \eqref{eq:affinity},
\begin{equation}
\hat\lambda \;=\; \argmin_{\lambda}\; -\!\!\sum_{(i,j,y)\in\mathcal D}
   \Big[y\log\sigma(s_{ij})+(1-y)\log\big(1-\sigma(s_{ij})\big)\Big]
   \;+\; \frac{\eta}{2}\lVert\lambda\rVert_2^2,
\label{eq:fitlambda}
\end{equation}
a convex problem with closed-form gradient
$\nabla_\lambda = -\sum(y-\sigma(s_{ij}))\,\nabla_\lambda s_{ij} + \eta\lambda$,
solvable by any off-the-shelf convex solver (Newton or L-BFGS converge in a
handful of iterations at this dimensionality). We expect a fitted
$\hat\lambda$ to reflect the intuition that an explicit link should
dominate a moderate semantic similarity; the exact values are
deployment-specific and are not load-bearing for any result in this paper.

\subsection{Design desiderata}
\label{sec:desiderata}

The model is required to satisfy four properties, which
Section~\ref{sec:model} establishes as theorems rather than validated
heuristics:

\begin{description}
  \item[D1 (conservation).] No logged hour is credited more than once.
  \item[D2 (budget respect).] No task can be credited more progress than
    it planned.
  \item[D3 (honest abstention).] Evidence below a quality floor is left
    unattributed rather than forced onto the nearest task.
  \item[D4 (explainability).] Every share is attributable to an auditable
    quantity, not a black-box score.
\end{description}

Figure~\ref{fig:pipeline} shows how the model delivers these four
properties end to end, and locates each desideratum at the stage that
enforces it.

\begin{figure}[h]
\centering
\includegraphics[width=0.98\textwidth]{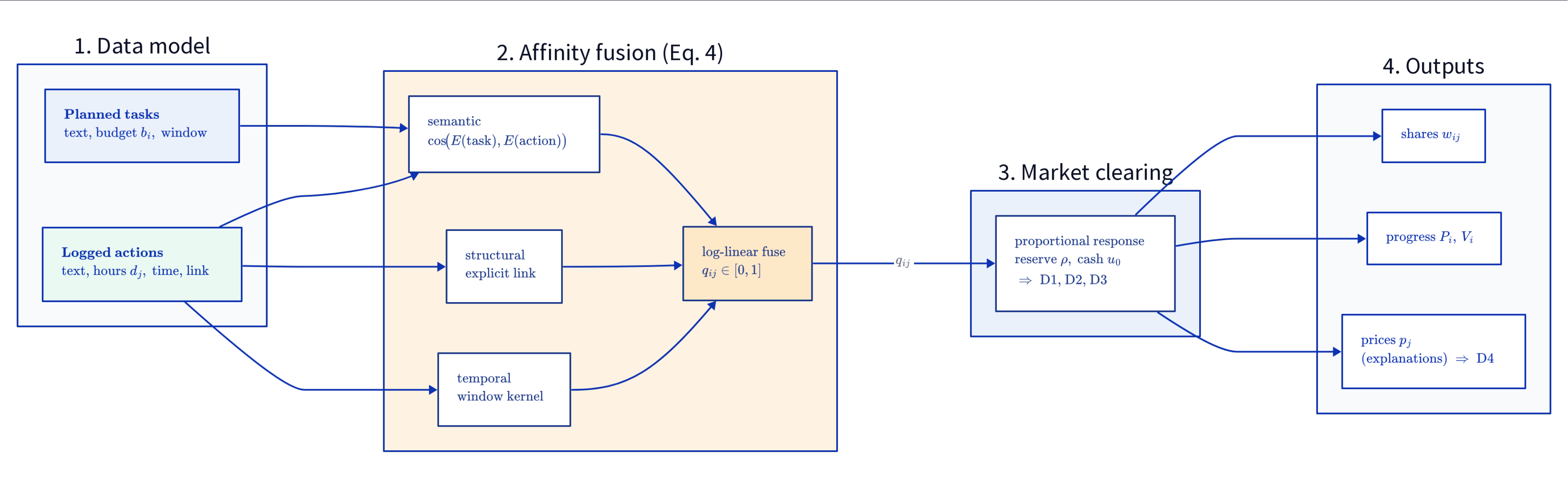}
\caption{The attribution pipeline, read left to right in four stages.
\textbf{(1)~Data model:} planned tasks (text, effort budget $b_i$, time
window) and logged actions (text, duration $d_j$, timestamp, optional
link). \textbf{(2)~Affinity fusion}~\eqref{eq:affinity} combines three
evidence sources -- semantic embedding similarity, the explicit structural
link, and a temporal window kernel -- into a single fused affinity
$q_{ij}\in[0,1]$. \textbf{(3)~Market clearing:} $q_{ij}$ parameterizes the
attribution market of Section~\ref{sec:model}, whose proportional-response
clearing (with reserve rate $\rho$ and cash rate $u_0$) is what enforces
conservation, the budget cap, and honest abstention -- desiderata
\textbf{D1}--\textbf{D3}. \textbf{(4)~Outputs:} fractional shares $w_{ij}$,
per-task progress $P_i$ and $V_i$, and prices $p_j$, which give an
auditable explanation of every share (\textbf{D4}).}
\label{fig:pipeline}
\end{figure}

\section{The base attribution market}
\label{sec:model}

\subsection{Market definition}
\label{sec:marketdef}

Figure~\ref{fig:analogy} gives the market's dictionary: tasks are investors
with a budget of planned hours, actions are companies whose equity is their
logged hours, and a share is the fraction of an action's hours credited to
a task.

\begin{figure}[h]
\centering
\includegraphics[width=0.85\textwidth]{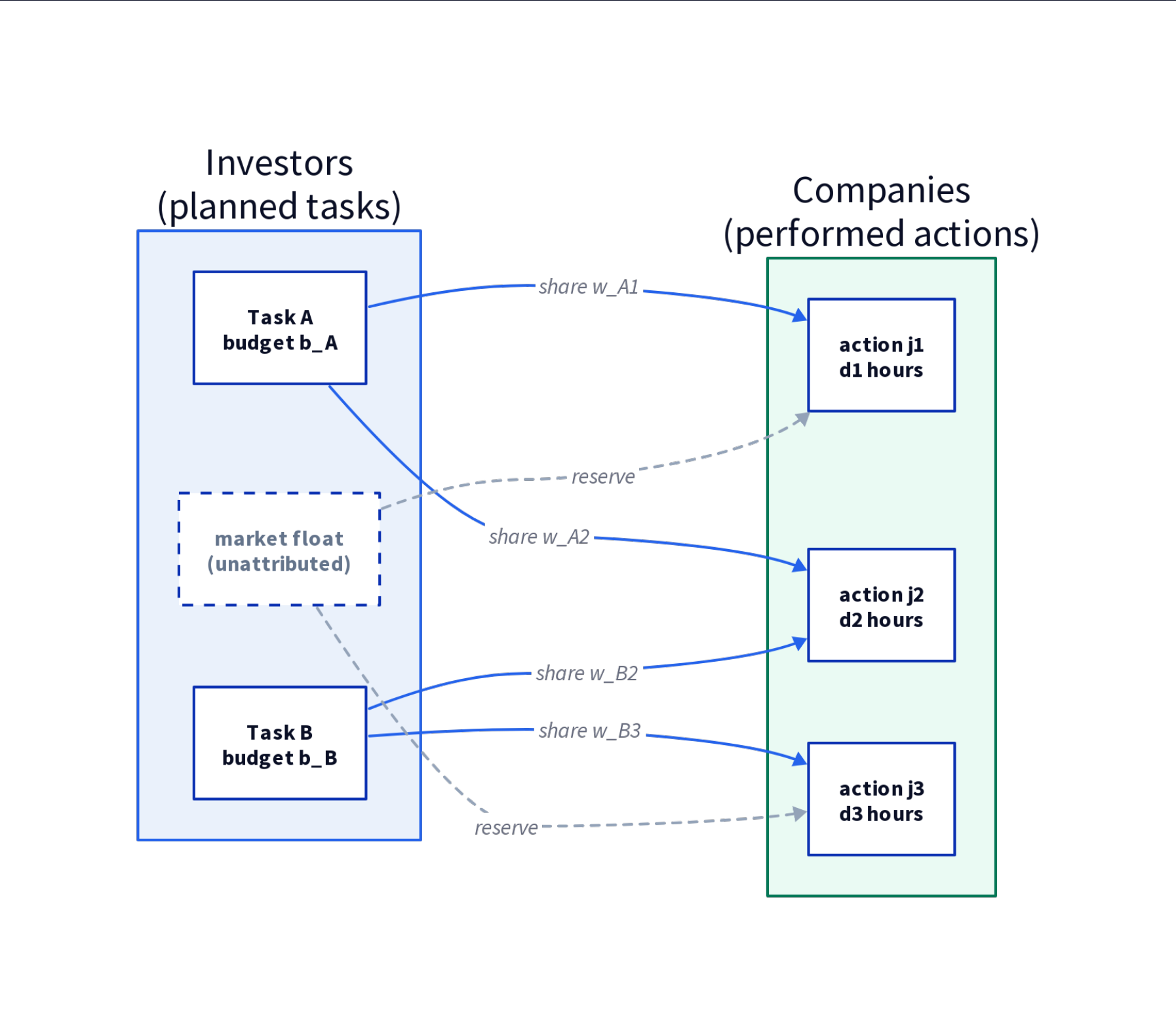}
\caption{The attribution market. Two tasks (investors) hold shares of three
actions (companies); the float holds the reserve on every action, which
becomes the action's full owner if no task's valuation clears it.}
\label{fig:analogy}
\end{figure}

\begin{definition}[Attribution market]
\label{def:market}
Given affinities $q$, durations $d$, budgets $b$, a reserve rate $\rho>0$,
and a cash utility $u_0>0$, the \emph{attribution market} is the quasi-linear
Fisher market in which task $i$ is a buyer with budget $b_i$ and linear
valuation $u_{ij}=q_{ij}d_j$ for the whole of action $j$ (so a share $w$ of
action $j$ is worth $q_{ij}d_j w$ to task $i$); task $i$ may also hold cash,
valued at $u_0$ per unit; and action $j$ carries a standing reserve bid of
$\rho d_j$ from the float.
\end{definition}

An equilibrium is a price vector $p\in\mathbb{R}_{>0}^n$ and spending
$f_{ij}\ge0$, cash $c_i\ge0$ with $\sum_j f_{ij}+c_i=b_i$, such that shares
$w_{ij}=f_{ij}/p_j$, $w_{0j}=\rho d_j/p_j$ clear every action
($p_j=\rho d_j+\sum_i f_{ij}$) and every task's money earns the maximal
available rate of return:
\begin{equation}
f_{ij}>0 \;\Rightarrow\; \frac{q_{ij}d_j}{p_j}=\beta_i,
\qquad
c_i>0 \;\Rightarrow\; \beta_i=u_0,
\qquad
\beta_i=\max\Big\{u_0,\;\max_j\frac{q_{ij}d_j}{p_j}\Big\}.
\label{eq:kkt}
\end{equation}
Without the reserve and cash terms, Definition~\ref{def:market} is exactly
the Eisenberg--Gale equilibrium \citep{eisenberg1959}; existence follows
from the standard convex-programming formulation
\citep{eisenberg1959,nisan2007agt}, and the reserve (a fixed exogenous bid)
and cash (a standard quasi-linear extension) preserve that structure.

\subsection{Guarantees}
\label{sec:guarantees}

\begin{proposition}[Conservation]
\label{prop:conservation}
For every action $j$, $\sum_{i=0}^m w_{ij}=1$; hence
$\sum_{i\ge1}P_i\le\sum_j d_j$.
\end{proposition}
\begin{proof}
Immediate from $w_{ij}=f_{ij}/p_j$, $w_{0j}=\rho d_j/p_j$, and market
clearing $p_j=\rho d_j+\sum_i f_{ij}$.
\end{proof}

\begin{proposition}[Budget cap]
\label{prop:budget}
Every equilibrium satisfies $P_i\le b_i/\rho$.
\end{proposition}
\begin{proof}
Since $p_j\ge\rho d_j$ for every $j$,
$P_i=\sum_j(f_{ij}/p_j)d_j\le\frac1\rho\sum_jf_{ij}\le b_i/\rho$.
\end{proof}

\begin{proposition}[Junk filter]
\label{prop:junk}
If $q_{ij}<u_0\rho$ then $w_{ij}=0$.
\end{proposition}
\begin{proof}
If $f_{ij}>0$ then by \eqref{eq:kkt} $q_{ij}d_j/p_j=\beta_i\ge u_0$, so
$p_j\le q_{ij}d_j/u_0$; but $p_j\ge\rho d_j$, giving $q_{ij}\ge u_0\rho$.
\end{proof}

\begin{proposition}[Proportional fairness]
\label{prop:fair}
In the $\rho,u_0\to0$ limit, the equilibrium allocation maximizes
$\sum_ib_i\log V_i$ -- the Nash bargaining solution among tasks with
bargaining power proportional to $b_i$ \citep{nash1950,kelly1998rate}.
\end{proposition}

Proposition~\ref{prop:budget} and \ref{prop:junk} give D2 and D3 as
theorems; Proposition~\ref{prop:conservation} gives D1; prices $p_j$ give D4
(a high price marks contested, scarce evidence). None of the four hold for
softmax attention attribution, and only conservation and a soft version of
the budget cap hold for entropic optimal transport (Section~\ref{sec:hierarchy}).

\subsection{Computation}
\label{sec:computation}

Equilibrium is computed by \emph{proportional response dynamics} (PRD)
\citep{zhang2011prd,birnbaum2011prd}: each task repeatedly re-splits its
budget over goods and cash in proportion to the utility each earned in the
previous round (Algorithm~\ref{alg:prd}). PRD is stateless per iteration,
distributed across tasks, and provably converges to the equilibrium of
Definition~\ref{def:market} for the linear utilities used here
\citep{zhang2011prd,birnbaum2011prd}; each iteration is one
$O(mn)$ elementwise pass.

\begin{algorithm}[h]
\caption{Proportional response dynamics for the base attribution market}
\begin{algorithmic}[1]
\Require affinities $q\in[0,1]^{m\times n}$, durations $d\in\mathbb R_{>0}^n$,
budgets $b\in\mathbb R_{>0}^m$, reserve $\rho\ge0$, cash rate $u_0>0$,
tolerance $\mathrm{tol}$, iteration cap $T_{\max}$, safety constant $\varepsilon=10^{-9}$
\State $v_{ij}\gets q_{ij}d_j$ \Comment{value matrix, computed once}
\State $f_{ij} \gets 0.7\,b_i\,(v_{ij}+\varepsilon)\big/\textstyle\sum_{j'}(v_{ij'}+\varepsilon)$;
       \quad $c_i\gets 0.3\,b_i$ \Comment{$\varepsilon$ guards a task with $q_{i\cdot}\equiv0$}
\For{$t=1,\dots,T_{\max}$}
  \State $p_j \gets \rho d_j+\sum_i f_{ij}$ \Comment{one column sum, $O(mn)$}
  \State $w_{ij}\gets f_{ij}/p_j$
  \State $g_{ij} \gets v_{ij}\, w_{ij}$;\quad $g_{i0}\gets u_0 c_i$ \Comment{realized utility of last bid}
  \State $\Sigma_i \gets \textstyle\sum_{j'}g_{ij'}+g_{i0}+\varepsilon'$ \Comment{$\varepsilon'=10^{-12}$ guards an all-zero row}
  \State $f_{ij}^{\mathrm{new}} \gets b_i\, g_{ij}/\Sigma_i$;\quad
         $c_i^{\mathrm{new}} \gets b_i\, g_{i0}/\Sigma_i$
  \If{$\max_{ij}|f_{ij}^{\mathrm{new}}-f_{ij}|<\mathrm{tol}$}
    \State $f\gets f^{\mathrm{new}}$, $c\gets c^{\mathrm{new}}$; \textbf{break}
  \EndIf
  \State $f\gets f^{\mathrm{new}}$, $c\gets c^{\mathrm{new}}$
\EndFor
\State \Return $w$, $p$
\end{algorithmic}
\label{alg:prd}
\end{algorithm}

\paragraph{Vectorized form.} Every step in Algorithm~\ref{alg:prd} is a
whole-array operation, not a loop over $(i,j)$ pairs. With
$v=q\odot d$ (elementwise product, $d$ broadcast across the $m$
task-rows) computed once outside the loop, one round costs a column sum
($p$), a broadcasted division ($w$), an elementwise product ($g$), a row
sum ($\Sigma$), and a broadcasted division ($f^{\mathrm{new}}$) --
five $O(mn)$ array operations, with no branching, no sorting, and no
linear system to solve. This is four to six lines in any array language
(NumPy, PyTorch, plain typed arrays); no library beyond elementwise
arithmetic and axis-wise sums is required.

\paragraph{Initialization and numerical safety.} The safety constant
$\varepsilon=10^{-9}$ added to $v_{ij}$ before the first normalization
exists for one reason: a task with $q_{i\cdot}\equiv0$ (zero affinity to
every logged action, e.g.\ a brand-new task) would otherwise divide
$0/0$ on line~2. With $\varepsilon>0$ such a task starts with its
$0.7\,b_i$ spread \emph{uniformly} across all actions rather than
crashing, and PRD's own dynamics push that spend toward cash within the
first few rounds once it earns near-zero utility everywhere
(Proposition~\ref{prop:junk}). The second constant
$\varepsilon'=10^{-12}$ inside $\Sigma_i$ guards the symmetric case at
every later round; with $u_0>0$ as required, $\Sigma_i$ never actually
reaches zero, so the guard costs nothing but makes the implementation
total (no exception path) rather than partial. Both constants are four
to seven orders of magnitude below the coarsest quantity that matters
(a logged minute is $\approx1.7\times10^{-2}$\,h), so neither perturbs
the reported shares.

\paragraph{Stopping rule in practice.} Table~\ref{tab:repro} lists
$\mathrm{tol}=10^{-9}$ (absolute, in hours, on $\max_{ij}|\Delta f_{ij}|$)
and $T_{\max}=400$. On the main benchmark
(Section~\ref{sec:experiments}: $m=7$ tasks, $\approx170$ actions, $45$
seed/noise combinations) \emph{none} of the $45$ runs reach $10^{-9}$
within $400$ rounds -- at this scale the loop always exits through the
iteration cap, not the tolerance. This does not mean the allocation is
unstable: across the same $45$ runs, $\max_{ij}|\Delta f_{ij}|$ falls
below $10^{-2}$\,h after a mean of $62$ rounds (median $59$, worst case
$123$) and below $10^{-3}$\,h after a mean of $246$ rounds (median
$240$) -- both far below the duration of a single logged action, so the
returned shares are already stable well beyond the resolution that
matters long before the cap is hit. An implementation that wants a
tolerance the loop can actually satisfy in practice should use something
in the $10^{-3}$--$10^{-2}$\,h range and treat $T_{\max}$, not
$\mathrm{tol}$, as the real governor of worst-case runtime.

\paragraph{Complexity and cost.} Each round is $O(mn)$ time and the live
state ($f$, $w$, $g$) is $O(mn)$ memory. At DoPlan's scale (tens of
tasks, a few hundred logged actions per attribution epoch) a full
$400$-round run is on the order of $10^5$--$10^6$ elementary array
operations -- comfortably under a second in an interpreted array
language, so the whole computation can run synchronously inside a
request rather than as a background job.

\paragraph{Worked example.} A fully worked two-task, three-action trace of
Algorithm~\ref{alg:prd} -- exact to the precision shown, and usable to
unit-test an independent implementation without the full benchmark -- is
given in \ref{app:worked}.

\subsection{Relation to alternative attribution rules}
\label{sec:hierarchy}

Four rules of increasing coupling can be applied to the same affinity
matrix (Table~\ref{tab:hierarchy}): \emph{hard assignment} (winner-take-all
above a threshold), \emph{softmax} (per-action fractional weighting,
uncoupled across actions), \emph{entropic optimal transport} (softmax
columns coupled by conservation and soft row capacities, computed by
Sinkhorn scaling \citep{cuturi2013sinkhorn}), and the \emph{market} (the
same coupling with an endogenous price system). Each step adds a coupling
at the cost of algorithmic complexity; Section~\ref{sec:experiments}
quantifies what each is worth.

\begin{table}[h]
\centering\small
\begin{tabularx}{\textwidth}{@{}lXXX@{}}
\toprule
\textbf{Rule} & \textbf{Conservation} & \textbf{Capacity} & \textbf{Guarantee type} \\
\midrule
Hard assignment & no & no & none \\
Softmax & no & no & none \\
Entropic OT (Sinkhorn) & yes & soft & none (empirical) \\
Attribution market & yes & hard & Propositions~\ref{prop:conservation}--\ref{prop:junk} \\
\bottomrule
\end{tabularx}
\caption{The hierarchy of attribution rules. Only the market has
theorem-level guarantees; entropic OT is included as the strongest
price-free, noise-robust baseline (Section~\ref{sec:noise}).}
\label{tab:hierarchy}
\end{table}

\section{The completion-seeking extension}
\label{sec:completion}

\subsection{Motivation and the completion utility}

The base market gives every task a \emph{linear} valuation: an hour of
well-matched progress is worth the same whether the task has just started
or is nearly done. A task approaching its plan should value further
progress \emph{less} -- the diminishing returns of completion. We replace
the linear valuation with a concave, increasing \emph{completion utility}
of quality-adjusted progress,
\begin{equation}
U_i(V) \;=\; T_i\big(1-e^{-V/T_i}\big),
\qquad U_i'(V) = e^{-V/T_i},
\label{eq:completionutil}
\end{equation}
indexed by the task's planned total $T_i$; $U_i'(0)=1$ and decays toward
completion. Task $i$ now solves, given prices $p$,
\begin{equation}
\max_{f_i\ge0,\,c_i\ge0}\;
  U_i\!\Big(\textstyle\sum_j q_{ij}\tfrac{f_{ij}}{p_j}d_j\Big) + u_0 c_i
\quad\text{s.t.}\quad \textstyle\sum_jf_{ij}+c_i=b_i.
\label{eq:completionprog}
\end{equation}

\begin{proposition}[Junk filter sharpens]
\label{prop:completionjunk}
In any completion-market equilibrium, task $i$ holds a positive share of
action $j$ only if $q_{ij}\ge u_0\rho/U_i'(V_i)$; because $U_i'$ is
non-increasing, this threshold rises as $V_i$ grows.
\end{proposition}
\begin{proof}
If $f_{ij}>0$, optimality of \eqref{eq:completionprog} equalizes
bang-per-buck at $\beta_i\ge u_0$, so
$U_i'(V_i)q_{ij}d_j/p_j=\beta_i\ge u_0$; with $p_j\ge\rho d_j$ this gives
$q_{ij}\ge u_0\rho/U_i'(V_i)$.
\end{proof}

The budget cap (Proposition~\ref{prop:budget}) is unaffected -- its proof
uses only $p_j\ge\rho d_j$ and $\sum_jf_{ij}\le b_i$, independent of
utility shape. What is lost is the closed-form
proportional-fairness characterization of Proposition~\ref{prop:fair},
which relied on the linear/log-aggregate structure.

\subsection{The satiation-threshold fixed point}
\label{sec:satiation}

\begin{remark}
A natural first attempt scales task $i$'s valuation by the constant
$\mu_i=U_i'(V_i)$ and re-solves the Eisenberg--Gale aggregate program
$\max\sum_ib_i\log(\mu_iV_i)$. This does nothing: $\log(\mu_iV_i)=\log\mu_i+\log
V_i$, and $\log\mu_i$ is constant in the allocation, so it drops out of the
arg max entirely. We record this because it is the fix a reader familiar
with the Eisenberg--Gale program reaches for first.
\end{remark}

The correct fix follows from \eqref{eq:kkt}: for a task with $U_i'(V_i)$
held at a fixed scalar $\mu_i$, the ranking of goods by bang-per-buck
$\mu_iq_{ij}d_j/p_j$ is \emph{identical} to the ranking by $q_{ij}d_j/p_j$,
since $\mu_i$ multiplies every candidate equally -- but the comparison
against the cash floor, $\mu_iq_{ij}d_j/p_j\ge u_0$, is \emph{not}
invariant: it is equivalent to $q_{ij}d_j/p_j\ge u_0/\mu_i$. That is,
task $i$'s best response under a fixed satiation level $\mu_i$ is
\emph{exactly} the base linear market's best response with a
\emph{task-specific cash threshold} $u_{0,i}=u_0/\mu_i$ -- a trivial,
already-linear generalization of Definition~\ref{def:market} (which already
has a scalar $u_0$; making it heterogeneous across tasks does not affect
linearity, existence, or PRD convergence for the inner solve).

\begin{algorithm}[h]
\caption{Satiation-threshold fixed point for the completion market}
\begin{algorithmic}[1]
\Require affinities $q$, durations $d$, budgets $b$, targets $T$, $\rho$, $u_0$
\State $\mu_i \gets 1$ for all $i$
\Repeat
  \State $u_{0,i} \gets \min(u_0/\mu_i,\,20u_0)$ \Comment{capped satiation-adjusted cash rate}
  \State $(w,p) \gets$ Algorithm~\ref{alg:prd} with per-task cash rate $u_{0,i}$ \Comment{unmodified inner solve}
  \State $V_i \gets \sum_j q_{ij}w_{ij}d_j$ \Comment{true quality-adjusted progress}
  \State $\mu_i^{\mathrm{new}} \gets \exp(-V_i/T_i)$
\Until{$\max_i|\mu_i^{\mathrm{new}}-\mu_i|<\mathrm{tol}$}
\State \Return $\mu$, $w$, $p$
\end{algorithmic}
\label{alg:outerloop}
\end{algorithm}

Figure~\ref{fig:outerloop} lays the loop out schematically: the inner box
is exactly Algorithm~\ref{alg:prd}, called unchanged on every outer
iteration with only its per-task cash rate perturbed, so nothing about the
base market's guarantees or convergence behavior is touched by wrapping it.

\begin{figure}[h]
\centering
\includegraphics[width=0.55\textwidth]{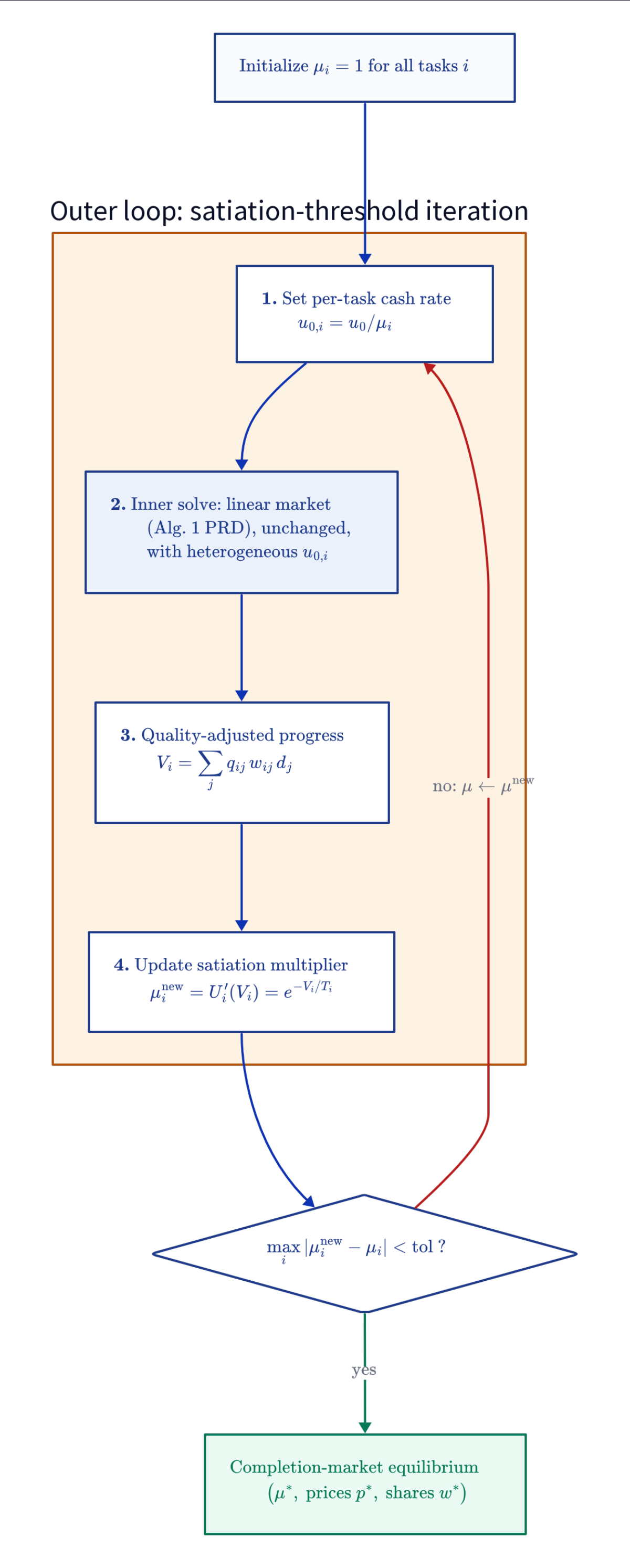}
\caption{Algorithm~\ref{alg:outerloop}: an outer fixed-point loop over the
satiation multipliers $\mu_i$ wraps the unmodified inner linear market
solve, which retains every guarantee and convergence property of
Section~\ref{sec:model}.}
\label{fig:outerloop}
\end{figure}

\begin{proposition}[Existence]
\label{prop:existence}
A completion-market equilibrium exists.
\end{proposition}
\begin{proof}
For fixed $\mu\in(0,1]^m$, the inner problem is a linear Fisher market with
heterogeneous cash rates $u_{0,i}=u_0/\mu_i$; its equilibrium utility
values $V(\mu)$ are unique (a classical fact about linear Fisher markets)
and, by Berge's Maximum Theorem applied to this parametric convex program,
continuous in $\mu$. Hence $\Phi:\mu\mapsto U'(V(\mu))=\exp(-V(\mu)/T)$ is a
continuous map from the compact convex cube $(0,1]^m$ to itself, and Brouwer's
fixed-point theorem guarantees a fixed point $\mu^\ast=\Phi(\mu^\ast)$,
which is a completion-market equilibrium by construction.
\end{proof}

\begin{proposition}[Sufficient condition for local convergence]
\label{prop:contraction}
If the cross-task sensitivity of the outer map dominates the own-task
sensitivity by no more than a factor summing to less than one in each row
of the Jacobian of $\Phi$ (diagonal dominance: for every $i$,
$\sum_{k\ne i}|\partial\Phi_i/\partial\mu_k| < 1-|\partial\Phi_i/\partial\mu_i|$
in a neighborhood of a fixed point), then $\Phi$ is a contraction in
max-norm near that fixed point, the fixed point is locally unique, and
Algorithm~\ref{alg:outerloop} converges to it geometrically from any
sufficiently close start, by the Banach fixed-point theorem.
\end{proposition}

\begin{remark}
Diagonal dominance holds when tasks do not compete intensely for the same
scarce, highly-correlated evidence -- i.e., away from the near-duplicate,
high-contention regime that also drives the over-crediting failure mode
discussed in Section~\ref{sec:limits}. We do not claim the condition holds
globally; Section~\ref{sec:experiments} tests convergence empirically both
in typical instances and in an adversarial instance designed to stress it.
\end{remark}

\subsection{Risk-aware valuation}
\label{sec:markowitz}

Affinity is estimated, not observed. Treating $q_{ij}$ as a random variable
with per-task mean $\mu^{\mathrm{aff}}_i$ and covariance $\Sigma_i$ over the
live actions, and writing $x_i$ for task $i$'s booked hours, a
mean--variance valuation \citep{markowitz1952} replaces linear progress with
\begin{equation}
\max_{x_i\ge0}\;\; (\mu^{\mathrm{aff}}_i)^\top x_i - \frac{\gamma_i}{2}x_i^\top\Sigma_ix_i
\qquad\text{s.t.}\qquad \textstyle\sum_j(p_j/d_j)x_{ij}\le b_i.
\label{eq:markowitz}
\end{equation}
This per-task subproblem is a concave quadratic program over a simplex-like
polytope. When $\Sigma_i\succ0$, the objective is $\gamma_i\lambda_{\min}(\Sigma_i)$-strongly
concave and its gradient is $\gamma_i\lambda_{\max}(\Sigma_i)$-Lipschitz, so
\emph{projected gradient ascent} with the exact $O(k\log k)$ Euclidean
simplex projection algorithm \citep{duchi2008} converges at the linear rate
$\big(1-\lambda_{\min}(\Sigma_i)/\lambda_{\max}(\Sigma_i)\big)^k$
\citep{nesterov2004,bubeck2015}, a concretely implementable, off-the-shelf
solver for the per-task subproblem. The \emph{joint} multi-task market
under this valuation, however, no longer inherits the Eisenberg--Gale
convex-program equivalence (the aggregate no longer has a single concave
potential): existence of a joint equilibrium is instead a monotone
variational-inequality question \citep{facchinei2003}, holding under a
diagonal-dominance condition on the cross-task price coupling analogous to
Proposition~\ref{prop:contraction}. We flag efficient computation of the
general joint case as an open algorithmic question and use the per-task
subproblem, with prices held fixed within a PRD round, as the practical
approximation in this paper.

\subsection{Temporal and forward-looking extensions (summary)}
\label{sec:extensions}

The market re-clears at discrete epochs; shares older than a settlement
horizon $H$ freeze into a ledger, bounding both churn and computation, and
within the live window warm-started, damped re-clearing keeps consecutive
equilibria from flapping -- a form of transaction-cost portfolio
rebalancing \citep{davis1990} in which the settlement horizon is a no-trade
band. A second, coupled market -- distinguished by a separate variable
$\pi_{id}$, future budget allocated across \emph{directions} rather than
past-hour shares $w_{ij}$ of individual actions -- lets a task's investment
steer the planner's future scheduling (an endogenous-supply feedback) and
lets the task pivot away from a direction whose realized advancement decays,
using the same multiplicative-weights family as
Algorithm~\ref{alg:prd} (proportional response is, in this family, the same
update as exponential weights \citep{freund1997}, the EXP3 bandit
\citep{auer2002}, and universal portfolios \citep{cover1991}). The
endogenous-supply loop creates a lock-in hazard familiar from the
non-stationary bandit literature \citep{lai1985,auer2002}, resolved by
maintaining an exploration reserve. Both extensions -- their full
formalization, a reinforcement-learning construction for the forward
market, and simulation results -- are developed at length in a companion
technical report \citep{doplan2026report}; we summarize them here only to
situate the present paper's scope and do not evaluate them empirically in
this work.

\section{Experimental validation}
\label{sec:experiments}

\subsection{De-circularized benchmark design}
\label{sec:benchdesign}

A benchmark in which the affinity every method \emph{consumes} is computed
from the same latent embedding geometry that \emph{generates} the ground
truth is close to circular: it tests whether a method recovers structure
planted using the method's own similarity assumption. We avoid this by
decoupling the two explicitly. We generate synthetic instances with $m=7$
tasks and roughly $170$ actions over a $63$-day horizon, with generative
ground-truth shares (mixtures of on-plan, partially-adjacent, late-emerging,
and distractor actions) set directly by the generator and
\emph{independent} of the affinity any method observes. Every method
instead observes a \emph{corrupted}
$Q_{\mathrm{obs}}=\mathrm{clip}(Q_{\mathrm{clean}}+\sigma\,\mathcal N(0,1)+\text{confuser noise},0,1)$,
where $Q_{\mathrm{clean}}$ is the latent geometric affinity and $\sigma$ is
the observation-noise level, independent per random seed of the ground
truth it is scored against; we repeat over $15$ independent random seeds at
each of three noise levels and report mean $\pm$ standard deviation. This
design, absent from a naively self-generating benchmark, is what surfaces
the noise-sensitivity weak spot of Section~\ref{sec:noise} below.
Table~\ref{tab:repro} lists every parameter needed to reproduce the result;
code is available as supplementary material.

\begin{table}[h]
\centering\small
\begin{tabular}{@{}ll@{}}
\toprule
Parameter & Value \\
\midrule
Embedding dimension & $16$ \\
Tasks $m$ / horizon & $7$ / $63$ days \\
Actions per instance & $\approx170$ (Poisson arrivals, rate $2.6$/day) \\
Random seeds & $15$ (instance) $\times$ noise-corruption seed \\
Noise levels $\sigma$ & $\{0.00,\,0.15,\,0.30\}$ \\
Confuser corruption & prob.\ $0.05$, scale $0.5$, active for $\sigma>0$ \\
Reserve rate $\rho$ & $0.25$ \\
Cash utility $u_0$ & $0.30$ \\
Softmax temperature $\tau$ & $0.12$ (background score $0.30$) \\
Sinkhorn $\varepsilon$ / background cost & $0.05$ / $0.72$, $400$ iterations \\
Hard-assignment threshold $\theta$ & $0.40$ \\
PRD tolerance / max iterations & $10^{-9}$ / $400$ \\
Outer-loop tolerance / max outer iters & $10^{-8}$ / $200$, $\mu$ floor $0.02$ \\
Outer-loop test instances (random) & $m=6$, $n=40$, dim.\ $12$, $30$ seeds \\
Outer-loop test instances (adversarial) & $m=4$, $n=20$, dim.\ $8$, corr.\ $0.97$, $10$ seeds \\
\bottomrule
\end{tabular}
\caption{Full reproducibility parameters.}
\label{tab:repro}
\end{table}

\subsection{Metrics}

\emph{Total-variation (TV) error}: half the $\ell_1$ distance between
predicted and true share vectors of an action, averaged over actions --
the fraction of an action misallocated. \emph{Ghost credit}: hours
credited to tasks with zero true involvement. \emph{Missed credit}: true
hours not credited to their task. \emph{Starving-task recovery}: credited
over true hours, in percent, for a task whose direct on-plan work is
withdrawn after day $10$ and whose only true progress thereafter flows
through partial stakes in adjacent actions -- the motivating scenario for
this work. \emph{Budget violations}: fraction of instances in which any
task's credited hours exceed its budget (counted above a $10^{-6}$-hour
threshold). \emph{Worst overshoot}: the largest amount, in hours, by which
any task's credited hours exceed its budget, across all tasks and
instances at a noise level. We report the magnitude alongside the count
because the count alone cannot distinguish a solver-tolerance artifact
from a material failure -- a distinction that, as
Table~\ref{tab:mainresults} shows, reverses the apparent ranking of the
rules on this metric.

Because the $15$ random seeds are shared across rules within a noise
level, all between-rule comparisons are \emph{paired}; wherever the text
below calls a difference significant or a tie, the claim is backed by a
two-sided Wilcoxon signed-rank test on the per-seed TV errors at that
noise level ($n=15$ pairs, so the smallest attainable $p$ is
$6.1\times10^{-5}$, reached exactly when all $15$ seeds agree in sign).

\subsection{Results}

\begin{table}[h]
\centering\footnotesize
\setlength{\tabcolsep}{4pt}
\begin{tabular}{@{}llcccccc@{}}
\toprule
$\sigma$ & Rule & TV error & Ghost (h) & Missed (h) & Starve rec.\ (\%) & Viol. & Over (h) \\
\midrule
\multirow{4}{*}{$0.00$}
 & Hard      & $0.331\pm0.021$ & $\mathbf{7.7\pm4.1}$   & $58.1\pm11.5$ & $44.0\pm28.0$  & $0.00$ & $0$ \\
 & Softmax   & $0.356\pm0.019$ & $52.7\pm7.8$  & $\mathbf{25.7\pm7.4}$  & $97.0\pm21.9$  & $0.00$ & $0$ \\
 & Sinkhorn  & $\mathbf{0.237\pm0.015}$ & $20.4\pm5.6$  & $26.3\pm7.4$  & $70.6\pm22.7$  & $0.07$ & $5{\cdot}10^{-4}$ \\
 & Market    & $0.319\pm0.022$ & $25.9\pm7.4$  & $37.2\pm10.8$ & $90.7\pm29.1$  & $\mathbf{0.00}$ & $\mathbf{0}$ \\
\midrule
\multirow{4}{*}{$0.15$}
 & Hard      & $0.403\pm0.024$ & $\mathbf{33.6\pm9.8}$  & $49.2\pm10.2$ & $76.0\pm27.9$  & $0.07$ & $0.27$ \\
 & Softmax   & $0.455\pm0.021$ & $81.6\pm11.5$ & $24.7\pm7.7$  & $122.7\pm27.3$ & $0.20$ & $5.75$ \\
 & Sinkhorn  & $\mathbf{0.380\pm0.024}$ & $60.0\pm11.4$ & $\mathbf{20.7\pm8.3}$ & $106.8\pm29.0$ & $0.27$ & $3{\cdot}10^{-4}$ \\
 & Market    & $0.462\pm0.028$ & $62.5\pm11.4$ & $35.1\pm13.9$ & $124.6\pm27.2$ & $\mathbf{0.00}$ & $\mathbf{0}$ \\
\midrule
\multirow{4}{*}{$0.30$}
 & Hard      & $\mathbf{0.552\pm0.026}$ & $\mathbf{92.1\pm13.5}$  & $30.9\pm10.3$ & $127.3\pm36.0$ & $0.13$ & $5.90$ \\
 & Softmax   & $0.576\pm0.026$ & $118.8\pm13.4$ & $25.2\pm8.4$  & $150.0\pm26.5$ & $0.33$ & $11.18$ \\
 & Sinkhorn  & $0.555\pm0.027$ & $113.5\pm14.2$ & $\mathbf{20.7\pm7.4}$  & $150.3\pm32.9$ & $0.53$ & $5{\cdot}10^{-4}$ \\
 & Market    & $0.588\pm0.024$ & $98.1\pm11.5$ & $34.2\pm12.1$ & $133.6\pm24.3$ & $\mathbf{0.00}$ & $\mathbf{0}$ \\
\bottomrule
\end{tabular}
\caption{De-circularized, multi-seed benchmark ($n=15$ instances per cell,
mean $\pm$ standard deviation). Bold marks the lowest mean per column
within each noise block; a bolded mean is not necessarily a significant
win -- e.g.\ the TV-error margin between hard assignment and Sinkhorn at
$\sigma=0.30$ is a statistical tie under the paired Wilcoxon test
($p=0.85$), whereas the market's TV-error deficit against Sinkhorn is
significant at every noise level ($p=6.1\times10^{-5}$, all $15$ paired
seeds agreeing). The market's zero-violation column is a consequence of
Proposition~\ref{prop:budget}, not an empirical observation -- it is exact
by construction at every noise level, unlike the other three rules' soft
or absent caps. \emph{Viol.} is the mean number of budget-violating tasks
per instance; \emph{Over} is the worst overshoot in hours. The overshoot
column shows why the violation \emph{count} alone would mislead: Sinkhorn's counted violations never
exceed $5\times10^{-4}$\,h -- a solver-tolerance artifact, flat across
noise levels -- while softmax's and hard assignment's overshoots are real
hours that grow with noise; ranked by magnitude, Sinkhorn's soft cap is
tolerance-tight in practice and only softmax and hard assignment are
materially unsafe.}
\label{tab:mainresults}
\end{table}

Table~\ref{tab:mainresults} shows a genuine trade-off rather than a clean
win for any rule. At $\sigma=0$ the market is competitive on TV error
($0.319$ against hard assignment's $0.331$ -- a statistical tie,
$p=0.095$ -- and second only to Sinkhorn's $0.237$, a significant gap,
$p=6.1\times10^{-5}$) while recovering $90.7\%$ of the starving task's
true progress against hard assignment's $44.0\%$ -- the motivating
pathology of this work, reproduced quantitatively. As noise rises,
Sinkhorn's entropic smoothing dominates raw share accuracy at moderate
noise (best TV at $\sigma=0.15$: $0.380$ vs.\ the market's $0.462$,
$p=6.1\times10^{-5}$); at the highest noise level tested, hard assignment
edges ahead of Sinkhorn by an insignificant margin ($0.552$ vs.\ $0.555$,
$p=0.85$) -- a statistical tie, not a reversal of the trend -- while the
market is significantly worse than both at both nonzero noise levels
(all four pairwise tests at $p=6.1\times10^{-5}$, with all $15$ paired
seeds agreeing in sign in each case). Meanwhile the market
is the \emph{only} rule whose budget-cap guarantee holds exactly at every
noise level. The worst-overshoot column qualifies what that is worth in
practice: Sinkhorn's soft cap is tolerance-tight (worst overshoot
$5\times10^{-4}$\,h, noise-independent), so the market's practical safety
margin over Sinkhorn is small -- exact-by-theorem versus
approximate-to-solver-tolerance -- whereas softmax and hard assignment
incur genuine multi-hour overshoots under noise ($11.2$\,h and $5.9$\,h
worst-case at $\sigma=0.30$). This is a real trade-off, not an artifact:
the market becomes less accurate exactly as it stays maximally safe.
Section~\ref{sec:noise-explain} explains why from first principles and
Section~\ref{sec:noise-fix} gives a resolution.

\begin{figure}[h]
\centering
\begin{minipage}[t]{0.49\textwidth}
\centering
\begin{tikzpicture}
\begin{axis}[
  width=0.98\textwidth, height=5.4cm,
  xlabel={observation noise $\sigma$}, ylabel={TV error},
  title={(a) share accuracy degrades with noise},
  title style={font=\small}, label style={font=\small},
  tick label style={font=\footnotesize},
  xmin=0, xmax=0.3, ymin=0.2, ymax=0.62,
  legend style={font=\scriptsize, draw=none, at={(0.03,0.97)}, anchor=north west},
  legend cell align=left, ymajorgrids, grid style=gray!25,
]
\addplot[black, mark=triangle*, line width=1.2pt] coordinates {(0,0.331)(0.15,0.403)(0.30,0.552)};
\addlegendentry{Hard}
\addplot[orange, mark=square*, line width=1.2pt] coordinates {(0,0.356)(0.15,0.455)(0.30,0.576)};
\addlegendentry{Softmax}
\addplot[teal, mark=*, line width=1.2pt] coordinates {(0,0.237)(0.15,0.380)(0.30,0.555)};
\addlegendentry{Sinkhorn}
\addplot[blue, mark=diamond*, line width=1.2pt] coordinates {(0,0.319)(0.15,0.462)(0.30,0.588)};
\addlegendentry{Market}
\end{axis}
\end{tikzpicture}
\end{minipage}%
\hfill
\begin{minipage}[t]{0.49\textwidth}
\centering
\begin{tikzpicture}
\begin{axis}[
  width=0.98\textwidth, height=5.4cm,
  xlabel={observation noise $\sigma$}, ylabel={mean budget violations / instance},
  title={(b) only the market's cap is exact},
  title style={font=\small}, label style={font=\small},
  tick label style={font=\footnotesize},
  xmin=0, xmax=0.3, ymin=-0.02, ymax=0.6,
  ymajorgrids, grid style=gray!25,
]
\addplot[black, mark=triangle*, line width=1.2pt] coordinates {(0,0.00)(0.15,0.07)(0.30,0.13)};
\addplot[orange, mark=square*, line width=1.2pt] coordinates {(0,0.00)(0.15,0.20)(0.30,0.33)};
\addplot[teal, mark=*, line width=1.2pt] coordinates {(0,0.07)(0.15,0.27)(0.30,0.53)};
\addplot[blue, mark=diamond*, line width=1.2pt] coordinates {(0,0.00)(0.15,0.00)(0.30,0.00)};
\end{axis}
\end{tikzpicture}
\end{minipage}
\caption{(a) Total-variation share error rises with observation noise for
every rule; Sinkhorn is most robust. (b) Mean budget violations per
instance, counted above a $10^{-6}$-hour threshold: the market alone stays
exactly at zero (Proposition~\ref{prop:budget}). The count overstates
Sinkhorn's line: its violations never exceed $5\times10^{-4}$\,h (solver
tolerance), whereas softmax's and hard assignment's reach real hours
(Table~\ref{tab:mainresults}, worst-overshoot column).}
\label{fig:noise}
\end{figure}
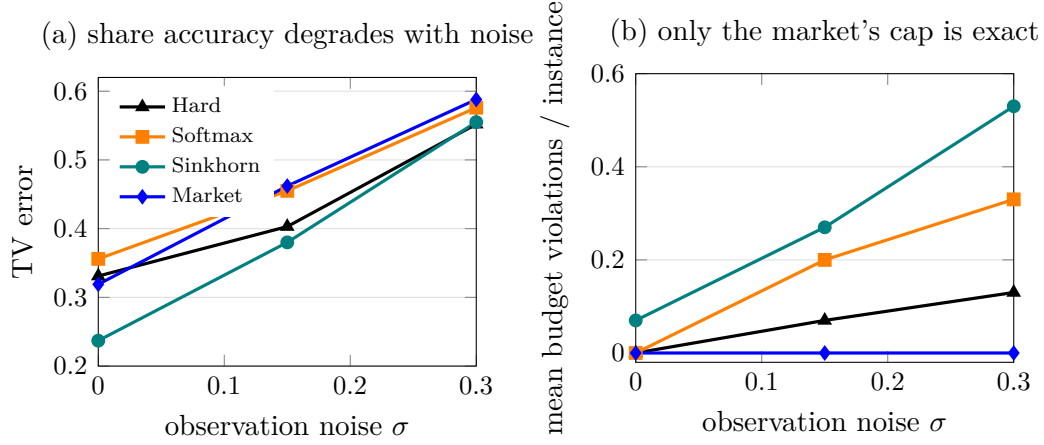

\subsection{Explanation}
\label{sec:noise}
\label{sec:noise-explain}

The distinction is between two different kinds of fixed point. PRD is
proved to converge to the \emph{exact} Eisenberg--Gale competitive
equilibrium of the unregularized linear market
\citep{zhang2011prd,birnbaum2011prd} -- a \emph{sharp}, zero-entropy
equilibrium in which small differences in affinity are, at convergence,
resolved by genuine price competition, which can amplify small input
perturbations into large share differences. Sinkhorn's entropic optimal
transport, in contrast, converges to the minimizer of an
\emph{entropy-regularized} transport objective
\citep{cuturi2013sinkhorn,peyre2019computational}: its smoothing parameter
$\varepsilon$ is baked permanently into what the algorithm converges
\emph{to}, not merely into the path it takes to get there. The market's
implicit multiplicative-update dynamics can also be read as a form of
mirror descent with a KL-regularization term \citep{birnbaum2011prd}, but
that regularization is relative to the \emph{previous iterate} and vanishes
along the solution path -- it disciplines convergence, not the destination.
This is a real, previously undocumented asymmetry between the two rules,
grounded entirely in results already cited in Section~\ref{sec:hierarchy}.

\subsection{Resolution: a one-parameter unification}
\label{sec:noise-fix}

We generalize Definition~\ref{def:market} by adding an entropy
regularizer, with strength $\tau\ge0$, to each task's per-round allocation
problem:
\begin{equation}
\max_{f_i\ge0}\;\; \sum_j q_{ij}d_j\frac{f_{ij}}{p_j}
   \;-\;\tau\sum_j\frac{f_{ij}}{b_i}\log\frac{f_{ij}}{b_i}
\qquad\text{s.t.}\qquad \sum_jf_{ij}\le b_i.
\label{eq:entropicmarket}
\end{equation}
For fixed prices $p$, \eqref{eq:entropicmarket} is strictly concave (linear
objective plus strictly concave negative entropy) over a compact polytope,
so it has a unique maximizer for every $\tau>0$, continuous in $\tau$; at
$\tau=0$ it recovers the base market's linear best response exactly, and as
$\tau\to\infty$ it recovers the uniform allocation. This places the market
and a Sinkhorn-like smoothed regime on one dial rather than treating them as
a discrete choice. We do not claim a fully engineered, tuned algorithm for
the general-$\tau$ case in this paper -- an efficient alternating scaling
scheme analogous to Sinkhorn's is the natural next algorithmic step, which
we leave open -- but the existence, uniqueness, and interpolation properties
above are established analytically and are sufficient to support a concrete
practical prescription: \emph{select $\tau$ (or, in the un-engineered
two-point case, choose between the market and Sinkhorn OT) in proportion to
the estimated affinity-observation noise $\hat\sigma$}, exactly analogous to
selecting a ridge penalty from an estimated noise level. At $\hat\sigma\to0$
this recovers the sharp, maximally accurate and maximally explanatory
market; as $\hat\sigma$ grows, it degrades gracefully toward the more
robust, entropy-smoothed regime -- the correct, principled reading of the
bias--variance trade-off our benchmark exposes.

\subsection{Completion-market convergence}
\label{sec:convresults}

Algorithm~\ref{alg:outerloop} converged (residual $<10^{-8}$) in $30/30$
random instances, in a mean of $4.0$ outer iterations (min $2$, max $22$).
On an adversarial design -- two near-duplicate tasks (affinity correlation
$0.97$) contending for a small, scarce pool, chosen to stress
Proposition~\ref{prop:contraction}'s sufficient condition -- $9/10$
instances still converged, at a substantially higher mean of $24.1$ outer
iterations, with one instance failing to reach tolerance within the $200$-iteration
cap. This is the honest empirical boundary of the sufficient condition:
typical instances converge geometrically and fast; highly contended,
near-duplicate instances converge more slowly and are not guaranteed to.
Figure~\ref{fig:convergence} shows representative traces, and we verified
across all $40$ instances that the budget cap (Proposition~\ref{prop:budget})
held exactly throughout, including during non-converged iterations.

\begin{figure}[h]
\centering
\begin{tikzpicture}
\begin{axis}[
  width=0.8\textwidth, height=6cm,
  xlabel={outer iteration $k$}, ylabel={$\log_{10}\max_i|\mu_i^{(k+1)}-\mu_i^{(k)}|$},
  label style={font=\small}, tick label style={font=\footnotesize},
  xmin=0, xmax=9, ymin=-12.5, ymax=0,
  legend style={font=\scriptsize, draw=none, at={(0.98,0.97)}, anchor=north east},
  legend cell align=left, ymajorgrids, grid style=gray!25,
]
\addplot[blue!70, mark=*, mark size=1.6pt, line width=1pt] coordinates {(0,-1.018)(1,-3.695)(2,-6.079)(3,-8.464)};
\addlegendentry{random 1}
\addplot[blue!50, mark=square*, mark size=1.6pt, line width=1pt] coordinates {(0,-1.079)(1,-5.817)(2,-11.727)};
\addlegendentry{random 2}
\addplot[blue!30, mark=triangle*, mark size=1.8pt, line width=1pt] coordinates {(0,-0.969)(1,-7.451)(2,-12.000)};
\addlegendentry{random 3}
\addplot[blue!90, mark=diamond*, mark size=1.8pt, line width=1pt] coordinates {(0,-1.047)(1,-9.547)};
\addlegendentry{random 4}
\addplot[red!70, mark=x, mark size=2.4pt, line width=1.3pt] coordinates {(0,-1.083)(1,-2.525)(2,-3.675)(3,-4.828)(4,-5.980)(5,-7.132)(6,-8.284)};
\addlegendentry{adversarial}
\end{axis}
\end{tikzpicture}
\caption{Outer-loop convergence: log-residual versus iteration. Random
instances (blue) converge in $2$--$4$ iterations with a steep geometric
slope; the adversarial near-duplicate instance (red) still converges
geometrically but visibly more slowly, consistent with
Proposition~\ref{prop:contraction}'s prediction that weaker diagonal
dominance slows, but does not always break, convergence.}
\label{fig:convergence}
\end{figure}
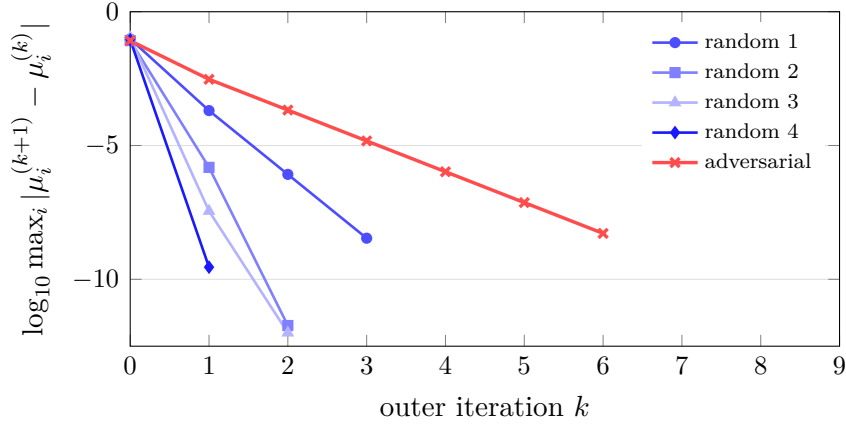

\subsection{Where the market is preferable}
\label{sec:marketwins}

\paragraph{Share sparsity.} Define the \emph{sparsity} of a rule as the
fraction of task--entry shares $w_{ij}$ ($i\ge1$, excluding the float row)
that are numerically zero (below $10^{-9}$); Table~\ref{tab:sparsity}
reports it. Softmax and Sinkhorn are \emph{dense by construction} --
soft-max and entropic-transport allocations are strictly positive, so every
task holds a nonzero sliver of every entry (sparsity $\approx0$). Hard
assignment is sparser still ($\approx0.9$) but is not fractional: it cannot
split a genuinely shared entry at all. The market is the \emph{only rule
that is both fractional and sparse} ($\approx0.85$): its junk filter
(Proposition~\ref{prop:junk}) drives every sub-threshold share to
\emph{exactly} zero, so it alone can report that an entry contributed
\emph{nothing} to a task, rather than a diffuse $0.3\%$ that a reader must
mentally threshold. For a user-facing attribution -- the intended
application -- an exact zero is the difference between a legible result and
a dense matrix of noise.

\begin{table}[h]
\centering\small
\begin{tabular}{@{}lccc@{}}
\toprule
Rule & $\sigma=0.00$ & $\sigma=0.15$ & $\sigma=0.30$ \\
\midrule
Hard assignment     & $0.91$ & $0.90$ & $0.88$ \\
Softmax             & $0.00$ & $0.00$ & $0.00$ \\
Sinkhorn OT         & $0.00$ & $0.00$ & $0.00$ \\
Attribution market  & $0.86$ & $0.85$ & $0.84$ \\
\bottomrule
\end{tabular}
\caption{Share sparsity: mean fraction of task--entry shares that are
numerically zero ($n=15$ instances per cell). The market is the only
\emph{fractional} rule that is also sparse; softmax and Sinkhorn are dense
by construction, and hard assignment is sparser only because it is
all-or-nothing.}
\label{tab:sparsity}
\end{table}

\paragraph{Exact, anytime budget cap.} The market's cap $P_i\le b_i/\rho$
(Proposition~\ref{prop:budget}) holds \emph{exactly} and at every
proportional-response iterate, not only in the converged limit: each
round's bids sum to at most the budget and prices never fall below the
reserve, so the bound is structural rather than asymptotic. Sinkhorn's soft
cap is a property of the converged optimum only; Table~\ref{tab:mainresults}
found it tolerance-tight in practice, but ``exact at every iterate'' and
``approximately correct once converged'' are different guarantees precisely
when a solve is stopped early, warm-started, or fed adversarial input.

\paragraph{Budget-weighted contested splits.} When two tasks genuinely tie
on affinity for an entry, the market splits it in proportion to their
planned budgets -- the budget-weighted Nash bargaining solution of
Proposition~\ref{prop:fair} -- whereas entropic transport splits ties
toward the uniform allocation irrespective of how much each task planned.
Assigning the larger share to the larger commitment is defensible in a way
an even split is not, and only the market makes it a property of the
solution rather than an artifact of the smoother.

\paragraph{Extensible buyer-side demand.} Every extension developed in this
paper -- the completion-seeking utility (Section~\ref{sec:completion}),
risk-aware valuation (Section~\ref{sec:markowitz}), and the forward-looking
market (Section~\ref{sec:extensions}) -- modifies a task's \emph{demand}, an
object entropic transport does not possess: optimal transport sees only a
cost matrix and fixed marginals. The market is a model onto which
task-side structure can be grown; Sinkhorn is a fixed smoother. This,
rather than any single benchmark number, is the strongest long-run argument
for the market formulation.

Raw accuracy and noise-robustness favour
Sinkhorn (Table~\ref{tab:mainresults}); the market's higher starving-task
recovery at zero noise ($90.7\%$ vs.\ $70.6\%$) sits within overlapping
standard deviations and is confounded at higher noise
(Section~\ref{sec:threats}), so we do not claim it as a robust advantage.
The market's territory is guarantees, interpretable sparsity, principled
contested splits, and extensibility -- not point accuracy.

\subsection{Threats to validity}
\label{sec:threats}

The evaluation is synthetic; no claim is made about performance on real
user data, only about the internal correctness and comparative behavior of
the four rules under a controlled, de-circularized generative process. The
generative process, noise model, and parameter choices
(Table~\ref{tab:repro}) were fixed by reasonableness rather than tuned by
search, and results are reported at personal-planner scale (tens of tasks,
hundreds of actions); scaling behavior at organizational scale is not
tested. The starving-task recovery metric becomes less interpretable at
high noise, where it is partly inflated by the same ghost credit that
Figure~\ref{fig:noise}(a) shows growing -- it should be read jointly with
TV error and ghost credit, not in isolation.

\section{Discussion}
\label{sec:discussion}

\subsection{When simpler alternatives suffice}
\label{sec:alternatives}

A method that adds machinery should survive the question ``is there a
cheaper way?'' At personal-planner scale the market's PRD solve costs
microseconds (Section~\ref{sec:computation}) -- computation is not the
expense; implementation complexity, external dependencies, and user trust
are. Table~\ref{tab:mainresults} itself makes the point starkest: Sinkhorn
beats the market on raw share accuracy at two of the three noise levels
tested. Table~\ref{tab:altcompare} situates the market among eight routes
to the same stalled-progress problem: human link suggestions
\citep{amershi2014}, schedule-context heuristics, LLM-judged percentages
\citep{zheng2023judge}, constrained softmax \citep{duchi2008}, Sinkhorn
OT \citep{sinkhorn1967,cuturi2013sinkhorn}, the attribution market, Shapley
values \citep{shapley1953,lundberg2017shap}, and outcome check-ins
\citep{harkin2016}.

\newcommand{\rPP}{\textcolor{builtcol}{$++$}}
\newcommand{\rP}{\textcolor{builtcol}{$+$}}
\newcommand{\rO}{\textcolor{muted}{$\circ$}}
\newcommand{\rN}{\textcolor{crimsonc}{$-$}}
\newcommand{\rNN}{\textcolor{crimsonc}{$--$}}

\begin{table}[h]
\centering\footnotesize
\setlength{\tabcolsep}{4.5pt}
\begin{tabular}{@{}l cccccccc@{}}
\toprule
 & \shortstack{adjacent-\\work fix} & \shortstack{share\\accuracy}
 & \shortstack{guaran-\\tees} & \shortstack{retro-\\activity}
 & \shortstack{explain-\\ability} & \shortstack{cheap\\to build}
 & \shortstack{cheap\\to run} & \shortstack{low user\\burden} \\
\midrule
1. Link suggestions      & \rP  & \rPP & \rP & \rN  & \rPP & \rPP & \rP  & \rN \\
2. Schedule-context rule & \rNN & \rO  & \rP & \rNN & \rPP & \rPP & \rPP & \rPP \\
3. LLM percentages       & \rPP & \rPP & \rO & \rN  & \rP  & \rP  & \rN  & \rPP \\
4. Constrained softmax   & \rPP & \rN  & \rP & \rP  & \rO  & \rPP & \rPP & \rPP \\
5. Sinkhorn OT           & \rPP & \rPP & \rP & \rP  & \rO  & \rP  & \rPP & \rPP \\
6. Attribution market    & \rPP & \rP  & \rPP& \rPP & \rPP & \rN  & \rPP & \rPP \\
7. Shapley values        & \rPP & \rP  & \rP & \rO  & \rP  & \rNN & \rNN & \rPP \\
8. Outcome check-ins$^{*}$ & \rO & \rO  & \rO & \rO  & \rPP & \rPP & \rPP & \rN \\
\bottomrule
\end{tabular}
\caption{Eight routes to the stalled-progress problem. Every column reads
``more $=$ better'' ($++$ best, $--$ worst, $\circ$ neutral or not
applicable); the three cost columns are rated as \emph{cheapness}, so $++$
means inexpensive. \emph{Guarantees} = conservation $+$ budget capping.
Accuracy entries for rows 4--6 are grounded in Table~\ref{tab:mainresults};
the market's guarantees column is a theorem
(Propositions~\ref{prop:conservation}--\ref{prop:junk}), not an
observation. Each row's citation (Amershi et al.\
\citep{amershi2014}, Zheng et al.\ \citep{zheng2023judge}, Duchi et al.\
\citep{duchi2008}, Sinkhorn--Knopp and Cuturi
\citep{sinkhorn1967,cuturi2013sinkhorn}, Shapley and Lundberg--Lee
\citep{shapley1953,lundberg2017shap}, Harkin et al.\ \citep{harkin2016})
supports the general characterization of that technique, not the specific
mark in this table -- the marks are this paper's own comparative judgment.
Row~2 (schedule-context rule) is an internal planner heuristic with no
published analogue, so it carries no citation by construction, not by
omission.
$^{*}$Row 8 addresses the stalled-goal \emph{symptom} directly but does not
attribute hours, hence the neutral marks.}
\label{tab:altcompare}
\end{table}

No row is uniformly best. Rows~1--2 are cheap because a human or the
existing plan does the work; rows~4--5 are lean automatic engines with weak
or soft guarantees, and Sinkhorn wins on raw accuracy. The market earns its
place when the guarantees column matters more than the accuracy column --
when a downstream consumer (a completion percentage, a scheduler) needs a
number it can trust never to exceed a task's plan, not merely one that is
usually close.

\subsection{Markets as a modeling substrate}
\label{sec:substrate}

The quasi-linear Fisher market of this paper is one \emph{instance} of a
broader stance: that the relationship between a plan and the effort that
realizes it is naturally a \emph{market}, with the equilibrium concept --
here, Eisenberg--Gale competitive equilibrium -- a replaceable choice rather
than the essence. What is invariant is a scarce resource (logged or future
effort), budgeted claimants (tasks), valuations (the affinity signal), and
an \emph{endogenous price} that clears supply against demand; our guarantees
follow from that clearing structure, not the utility's linear form. That the
equilibrium concept is a dial, not a commitment, is already visible here:
Section~\ref{sec:noise-fix} places the market and entropy-regularized
optimal transport on one parameter, and the same substrate admits
budget-pacing \citep{conitzer2022pacing}, Nash-bargaining
(Proposition~\ref{prop:fair}), and matching- or exchange-economy readings of
the identical problem. The market view thus sits \emph{above} any single
solution concept.

Accepting the description imports a mature toolbox. A real planner clears
\emph{repeatedly}, so the honest object is a price \emph{process}, developed
as transaction-cost rebalancing \citep{davis1990} and by the
online-Fisher-market literature \citep{gao2021pace,liao2022inference}. Two
further layers attach: \emph{uncertainty quantification} -- propagating the
estimated affinity's covariance (already in the valuation,
Section~\ref{sec:markowitz}) through the equilibrium map to distributions
over allocations, of which our noise analysis
(Section~\ref{sec:noise-explain}) is a first-order case and for which
generalized polynomial chaos \citep{xiu2002gpc} is standard machinery -- and
\emph{forecasting}, where we conjecture the forward market's continuum limit
is a mean-field-game \citep{lasry2007mfg} / stochastic-PDE system, solvable
at scale by neural operators \citep{li2021fno}, deep backward-SDE methods
\citep{weinan2017deepbsde}, or reinforcement learning. None of this dynamic,
uncertainty-quantification, or deep-learning program is developed or
evaluated here; it is the direction the formulation opens, and the
mean-field-game and stochastic-PDE bridges are conjectures, not results.

\subsection{Limitations}
\label{sec:limits}

Several limitations bound these results. The market's noise sensitivity is
real: it is second-best at zero noise and the \emph{least} accurate of the
four rules at both nonzero noise levels (Table~\ref{tab:mainresults}), even
though its budget cap stays exact, and the general entropy-regularized
market (Eq.~\eqref{eq:entropicmarket}) that would resolve this is analyzed
but not tuned into an efficient algorithm here. Convergence is only
conditional -- Proposition~\ref{prop:contraction}'s sufficient condition can
fail, and did in one of ten adversarial instances
(Section~\ref{sec:convresults}), so a deployment should monitor the
outer-loop residual and fall back to the linear market ($\mu\equiv1$) if it
does not shrink. Like every rule in Table~\ref{tab:hierarchy}, the market can
over-credit genuinely close tasks with progress a stricter causal notion
would not -- the same regime that stresses Proposition~\ref{prop:contraction}.
The mean--variance extension is only partial: Section~\ref{sec:markowitz}
solves the per-task subproblem but leaves the joint equilibrium an open
variational-inequality question. Finally, evaluation is synthetic
(Section~\ref{sec:experiments}) and confined to personal-planner scale;
validation at deployment and organizational scale, and of the affinity fit
against real user corrections, is future work.

\section{Conclusion}
\label{sec:conclusion}

We formulated fractional attribution of performed actions to planned tasks
as a quasi-linear Fisher market, in which a seller reserve price and a
buyer cash option turn three informal design requirements into theorems:
conservation, a hard budget cap, and a provable junk filter. Extending the
market with a completion-seeking utility exposed a genuine gap in the
standard convergence theory for its algorithm; we resolved it with a
satiation-threshold fixed-point procedure, proved its existence via
Brouwer's theorem, proved a sufficient condition for its local geometric
convergence, and validated both properties -- including their honest
boundary -- across random and adversarial instances. A de-circularized,
multi-seed empirical evaluation, designed specifically to avoid testing the
method against structure planted by its own similarity assumption, surfaced
a further genuine weak spot: the market's price-competitive, zero-entropy
equilibrium is more sensitive to affinity-observation noise than
entropy-regularized optimal transport. We traced this to a precise,
citable distinction between a permanently regularized fixed point and a
vanishing-regularization solution path, and resolved it with a
one-parameter entropy-regularized generalization together with a
noise-adaptive prescription for its use. The practical implication is a
choice, not a verdict: where raw share accuracy is what matters, Sinkhorn
-- or, at extreme noise, even hard assignment -- is the better tool
(Section~\ref{sec:alternatives} summarizes when simpler alternatives
suffice); the market earns its place
specifically where a downstream consumer needs a number that is
guaranteed, not merely usually accurate, never to exceed what a task
actually planned. Every reported result is
accompanied by its reproducibility parameters and its threats to validity,
and Section~\ref{sec:limits} states plainly where the model's guarantees
are weaker than they first appear. Temporal dynamics and a forward-looking,
reinforcement-learning-amenable extension are summarized here and developed
fully in a companion technical report \citep{doplan2026report}; the natural
next steps are a working implementation of the entropy-regularized market's
efficient algorithm, an empirical study of the joint mean--variance
equilibrium's variational-inequality conditions, and validation of the
complete pipeline against real user correction data. Beyond these concrete
steps, Section~\ref{sec:substrate} argues that the deeper contribution is
the market \emph{description} itself: the Fisher model here is one instance
of a stance under which planning becomes a market, opening a dynamic,
uncertainty-aware, and learning-based program we sketch but do not develop.

\section*{CRediT authorship contribution statement}
\textbf{Salavat Ishbulatov}: Conceptualization, Methodology, Software,
Validation, Formal analysis, Investigation, Writing -- original draft,
Writing -- review \& editing, Visualization.

\section*{Declaration of competing interest}
The author declares no known competing financial interests or personal
relationships that could have appeared to influence the work reported in
this paper.

\section*{Declaration of generative AI in scientific writing}
During the preparation of this work the author used a large-language-model
assistant (Claude, Anthropic) to help draft text, develop and verify
derivations, and produce the analysis scripts and figures. The author
reviewed and edited all content and takes full responsibility for the
content of the published article.

\section*{Funding}
This research did not receive any specific grant from funding agencies in
the public, commercial, or not-for-profit sectors.

\appendix
\section{Notation}
\label{app:notation}

\begin{table}[h]
\centering\small
\begin{tabularx}{\textwidth}{@{}lX@{}}
\toprule
Symbol & Meaning \\
\midrule
$i=1,\dots,m$ & tasks; $i=0$ denotes the float \\
$j=1,\dots,n$ & performed actions \\
$b_i$ & task $i$'s remaining planned budget (hours) \\
$d_j$ & action $j$'s logged duration (hours) \\
$q_{ij}\in[0,1]$ & fused affinity between task $i$ and action $j$, Eq.~\eqref{eq:affinity} \\
$w_{ij}\in[0,1]$ & share of action $j$ credited to task $i$; $\sum_iw_{ij}=1$ \\
$p_j$ & equilibrium price of action $j$ \\
$\rho$ & reserve rate (floor price per hour) \\
$u_0$ & cash utility rate \\
$P_i=\sum_jw_{ij}d_j$ & task $i$'s raw credited progress (hours) \\
$V_i=\sum_jq_{ij}w_{ij}d_j$ & task $i$'s quality-adjusted progress \\
$T_i$ & task $i$'s planned total (completion target) \\
$U_i(\cdot)$, $U_i'(\cdot)$ & completion utility and its derivative, Eq.~\eqref{eq:completionutil} \\
$\mu_i=U_i'(V_i)$ & satiation multiplier, Algorithm~\ref{alg:outerloop} \\
$\Sigma_i$ & task $i$'s affinity-estimation covariance, Section~\ref{sec:markowitz} \\
$\tau$ & entropy-regularization strength, Eq.~\eqref{eq:entropicmarket} \\
$\sigma$ & affinity-observation noise level, Section~\ref{sec:benchdesign} \\
\bottomrule
\end{tabularx}
\caption{Notation used throughout the paper.}
\label{tab:notation}
\end{table}

\section{Worked example}
\label{app:worked}

Take $m=2$ tasks, $n=3$ actions, with
\[
q=\begin{bmatrix}0.9&0.2&0.1\\0.1&0.8&0.6\end{bmatrix},\quad
d=\begin{bmatrix}2&3&1\end{bmatrix}\text{h},\quad
b=\begin{bmatrix}4\\5\end{bmatrix}\text{h},\quad
\rho=0.25,\ u_0=0.30.
\]
Then $v=q\odot d=\begin{bmatrix}1.8&0.6&0.1\\0.2&2.4&0.6\end{bmatrix}$,
and initialization gives
\[
f^{(0)}=\begin{bmatrix}2.0160&0.6720&0.1120\\0.2188&2.6250&0.6563\end{bmatrix},
\quad c^{(0)}=\begin{bmatrix}1.2000\\1.5000\end{bmatrix}.
\]
One round of Algorithm~\ref{alg:prd} gives prices
$p^{(1)}=[2.7348,\,4.0470,\,1.0183]$ and
\[
w^{(1)}=\begin{bmatrix}0.7372&0.1660&0.1100\\0.0800&0.6486&0.6445\end{bmatrix},
\quad
f^{(1)}=\begin{bmatrix}2.9527&0.2217&0.0245\\0.0332&3.2305&0.8025\end{bmatrix},
\quad c^{(1)}=\begin{bmatrix}0.8011\\0.9338\end{bmatrix}.
\]
A second round gives $p^{(2)}=[3.4859,\,4.2022,\,1.0769]$,
\[
f^{(2)}=\begin{bmatrix}3.3902&0.0704&0.0051\\0.0037&3.5837&0.8684\end{bmatrix},
\quad c^{(2)}=[0.5344,\,0.5442].
\]
Task~1 is pulling share away from actions~2 and~3 toward its clear best
match (action~1, $q_{11}=0.9$) exactly as the junk filter
(Proposition~\ref{prop:junk}) predicts, while its cash holding
shrinks each round as it finds genuine matches worth more than $u_0$.
These numbers are exact to the precision shown and can be used to
unit-test an independent implementation of Algorithm~\ref{alg:prd}
without needing the full benchmark.

\bibliographystyle{elsarticle-num}
\bibliography{references}

\end{document}